\documentclass[final]{siamltex}

\usepackage[parfill]{parskip}    
\usepackage{amsmath,amssymb,latexsym,enumerate,mathrsfs}
\usepackage{hyperref}
\usepackage{array}
\usepackage{times}
\usepackage[]{mdframed}
\usepackage{epstopdf}
\usepackage{subfig}
\usepackage{graphicx}
\usepackage{hyperref} 
\usepackage{multicol}
\usepackage{framed}
\usepackage{moreverb}
\usepackage{marvosym}
\usepackage{color,soul}
\definecolor{lightblue}{rgb}{.90,.95,1}
\sethlcolor{yellow}

\title{{\Large Technical Note}\\ ~\\
A decision-theoretic model for a principal-agent collaborative learning problem}

\author{Getachew K. Befekadu}

\begin{document}
\maketitle

\renewcommand{\thefootnote}{\arabic{footnote}}

\begin{abstract}
In this technical note, we consider a collaborative learning framework with principal-agent setting, in which the principal at each time-step determines a set of appropriate aggregation coefficients based on how the current parameter estimates from a group of $K$ agents effectively performed in connection with a separate test dataset, which is not part of the agents' training model datasets. Whereas, the agents, who act together as a team, then update their parameter estimates using a discrete-time version of Langevin dynamics with mean-field-like interaction term, but guided by their respective different training model datasets. Here, we propose a decision-theoretic framework that explicitly describes how the principal progressively determines a set of nonnegative and sum to one aggregation coefficients used by the agents in their mean-field-like interaction term, that eventually leading them to reach a consensus optimal parameter estimate. Interestingly, due to the inherent feedbacks and cooperative behavior among the agents, the proposed framework offers some advantages in terms of stability and generalization, despite that both the principal and the agents do not necessarily need to have any knowledge of the sample distributions or the quality of each others' datasets.
\end{abstract}

\section{Introduction} \label{S1}
In this technical note, we consider a principal-agent collaborative learning framework, in which the principal progressively determines a set of nonnegative and sum to one aggregation coefficients based on how the current parameter estimates from a group of $N$ agents effectively performed in connection with a separate test dataset, that is not part of the agents' training model datasets. Whereas, the agents, who act together as a team, then update their parameter estimates using a discrete-time version of Langevin dynamics with mean-field-like interaction term, but guided by their respective different training model datasets. In particular, we propose a decision-theoretic model that allows the principal to determines a set of appropriate aggregation coefficients used by the agents in their mean-field-like interaction term, that eventually leading them to reach a consensus optimal parameter estimate. Moreover, due to the inherent feedbacks and cooperative behavior among the agents, the proposed collaborative learning offers some advantages in terms of stability and generalization, despite that both the principal and the agents do not necessarily need to have any knowledge of the sample distributions or the quality of each others' datasets. 

Here, it is worth remarking that the decision-theoretic model we proposed can provide new insights in the context of collaborative learning with principal-agent setting, although we acknowledge that there are a number of conceptual and theoretical challenges still need to be addressed. For example, the hypoellipticity property that we imposed on the Langevin dynamics (i.e., an assumption which is related to a strong accessibility property of controllable nonlinear systems driven by white noise) together with a suitable energy-like Lyapunov function, that verifying the stability property, is sufficient for the agents to reach an optimal estimate. However, constructing/finding such a Lyapunov function, that certifies an exponential convergence rate and further evidencing the speed of learning, is difficult due to weak dissipation in the moment direction.

The remainder of this technical note is organized as follows. In Section~\ref{S2}, we provide a formal problem statement, where we consider a collaborative learning framework with principal-agent setting that can be viewed as an extension for enhancing the learning performance or generalization. In particular, we present a decision-theoretic model that allows the principal-agent setting in the context of collaborative learning to exhibit some inherent feedbacks or cooperative behavior, that eventually leading the agents to reach a consensus optimal parameter estimate. Here, we also describe a generic algorithm that provides a viable way to implement such a collaborative learning framework. Finally, Section~\ref{S3} presents numerical results for a typical nonlinear regression problem along with some discussions.

\section{Problem formulation} \label{S2} In this section, we provide a formal problem statement, where we consider a collaborative learning framework, with principal-agent setting that exhibits some inherent feedbacks or cooperative behavior among the agents. In particular, the learning framework consists of the following core ideas:
\begin{enumerate} [(i).]
\item A total of $(K+1)$ datasets, i.e., $Z^{(k)} = \bigl\{ (x_i^{(k)}, y_i^{(k)})\bigr\}_{i=1}^{m_k}$, for $k=1,2, \ldots, K+1$, each with data size of $m_k$. The datasets $\bigl\{ Z^{(k)} \bigr\}_{k=1}^{K+1}$ may be generated from a given original dataset by means of bootstrapping with/without replacement or they may simply represent distributed/decentralized observations associated with point estimation problem in modeling of high-dimensional nonlinear functions. Moreover, we assume that the datasets $\bigl\{ Z^{(k)} \bigr\}_{k=1}^{K}$ will be used for model training purpose, corresponding to each of the $K$ agents, (i.e., numbered $k = 1, 2, \dots, K$), while the last dataset $Z^{(K+1)}$ will be used by the principal for quantifying how the agents' current parameter estimate effectively perform as part of the collaborative learning process.
 \item The agents are largely tasked to search for a parameter estimate $\theta \in \Gamma$, from a finite-dimensional parameter space $\mathbb{R}^p$, such that the function $h_{\theta(x)} \in \mathcal{H}$, i.e., from a given class of hypothesis function space $\mathcal{H}$, describes best the corresponding model training datasets $Z^{(k)}$, $k=1,2, \ldots, K$. Here, we further assume that each agent updates its parameter estimate using a discrete-time approximation of Langevin dynamics with mean-field-like interaction term, but guided by its respective training model dataset.
 \item The principal progressively determines a set of nonnegative and sum to one aggregation coefficients based on how the agents' current parameter estimates effectively performed in connection with the testing dataset $Z^{(K+1)}$. Here, in doing so, the principal uses a naive decision-theoretic model for computing the aggregation coefficients that quantify the way agents influence each other or exchange information among themselves at each time step through their mean-field-like interaction terms. 
\end{enumerate}

In terms of mathematical construct, searching for an optimal parameter $\theta^{\ast} \in \Gamma \subset \mathbb{R}^p$ can be associated with the method of aggregating the {\it steady-state solutions} to the following family of gradient systems, whose {\it time-evolutions} are guided by the corresponding traning datasets $\bigl\{Z^{(k)}\bigr\}_{k=1}^K$, i.e.,
\begin{align}
 \dot{\theta}^{(k)}(t) = - \nabla J_k(\theta^{(k)}(t),Z^{(k)}), \quad \theta^{(k)}(0) = \theta_0, \label{Eq2.1}
\end{align}
with $J_k(\theta^{(k)}, Z^{(k)}) = \frac{1}{m_k} \sum\nolimits_{i=1}^{m_k} {\ell}\bigl(h_{\theta^{(k)}}(x_i^{(k)}), y_i^{(k)}\bigr)$, where $\ell$ is a suitable loss function that quantifies the lack-of-fit between the model and the datasets. Note that such an approach has a series short coming due to the solution may become trapped at local minimum, rather than the global minimum solution. One way of addressing such difficulty is to consider a solution for a family of related stochastic differential equations (SDEs) with mean-field-like interaction term, where the mean-field-like interaction term is typically fostering an inherent feedbacks or collaborative behavior among the agents (e.g., see \cite{r1a} for general discussions on diffusions for global optimization problems; see also \cite{r1b} for additional discussions on learning via dynamical systems). Here, we also assume that $\nabla J_k(\theta,Z^{(k)})$, for each $k \in \{1,2, \ldots, K\}$, is uniformly Lipschitz and further satisfies the following growth condition
\begin{align}
 \bigl\vert \nabla J_k(\theta,Z^{(k)}) \bigr\vert^2 \le L_{\rm Lip} \bigl(1 + \vert \theta \vert^2 \bigr), \quad \forall\theta \in \Gamma \subset \mathbb{R}^p, \label{Eq2.2}
\end{align}
for some constant $L_{\rm Lip} > 0$.  

\subsection{Parameter-updating model for the agents} \label{S2.1} In what follows, we allow each agent to update its parameter estimate using a discrete-time approximation of the following Langevin dynamics with mean-field-like interaction term, i.e.,
\begin{align}
 d\Theta_t^{(k)} =& P_t^{(k)} dt \notag \\
 dP_t^{(k)} =& -\gamma P_t^{(k)} dt - \nabla J_k(\Theta_t^{(k)},Z^{(k)})dt - \eta \left(\Theta_t^{(k)} - \bar{\Theta}_t \right)dt  \notag \\
    & \quad\quad\quad\quad\quad\quad\quad ~~~~+ \left(c/\sqrt{\log(t + 2)}\right) I_p dW_t^{(k)}, \notag \\
   & \quad\quad \quad \quad \quad (\Theta_0^{(k)}, P_0^{(k)}) = (\theta_0^{(k)}, p_0^{(k)}), \quad k =1,2, \ldots, K, \label{Eq2.3}
\end{align}
where $c$ is small positive number, $\gamma > 0$ plays the role of a friction coefficient related to the moment equation, $I_p$ is a $p \times p$ identity matrix, and $W_t^{(k)}$ is a $p$-dimensional standard Wiener process. The term $\eta \bigl(\Theta_t^{(k)} - \bar{\Theta}_t \bigr)$, with $\eta > 0$, $\bar{\Theta}_t = \sum\nolimits_{k=1}^{K} \pi_t^{(k)} \Theta_t^{(k)}$, $\pi_t^{(k)} \ge 0$ and $\sum\nolimits_{k=1}^{K} \pi_t^{(k)} = 1$, can be viewed as a mean-field-like interaction creating a tendency for the agents to move toward a consensus solution and thus providing an inherent feedbacks and cooperative behavior among the agents. Moreover, if $dP_t^{(k)}$ is zero in the above coupled SDEs of Equation~\eqref{Eq2.3}, then system of equations will be reduced to standard gradient dynamics with small additive noise and mean-field-like interaction term. Note that a small additive noise enters only in the second equation and then it propagates to the first equation. Here, we assume that a hypoellipticity property must hold for the above coupled SDEs (e.g., see \cite{r2}, \cite{r3} or \cite{r4} for additional discussions on diffusion processes arising from controllable systems), which is in general related to a strong accessibility property of controllable nonlinear systems that are driven by white noises (e.g., see \cite{r5} concerning the controllability of nonlinear systems, which is closely related to works of Stroock and Varadhan \cite{r6} and Ichihara and \cite{r7}; see also \cite[Section~3]{r8} and \cite[Theorem~2]{r3}). That is, the hypoellipticity assumption further implies that the diffusion process $(\Theta_t^{(k)}, P_t^{(k)})$ has a transition probability density function with strong Feller property. Moreover, we can construct a suitable energy-like Lyapunov function for the system of equations in Equation~\eqref{Eq2.3}, that verifies the stability property, is sufficient for the agents to eventually reach a consensus optimal estimate (e.g., see \cite{r9} for additional discussions on the stability of stochastic dynamical systems).

Then, for an equidistant discretization time $\delta=\tau_{n+1} - \tau_n = T/N$, $n=0,1,2, \ldots, N-1$, with $0=\tau_0 < \tau_1< \ldots < \tau_n<\ldots<\tau_N=T$, of the time interval $[0,T]$, the {\it Euler-Maruyama} approximation for the continuous-time stochastic processes $\bigl(\Theta^{(k)}, P^{(k)}\bigr) =\bigl\{ \bigl(\Theta_t^{(k)},P_t^{(k)}\bigr) \, 0 \le t \le T \bigr\}$, for $k=1,2, \ldots, K$, satisfying the following iterative scheme     
\begin{align}
 \Theta_{n+1}^{(k)} &= \Theta_{n}^{(k)} + \delta P_{n}^{(k)} \notag \\
 P_{n+1}^{(k)} &= (1 - \delta \gamma) P_{n}^{(k)} - \delta \nabla J_k(\Theta_n^{(k)},\hat{Z}^{(k)}) - \delta \eta \left(\Theta_n^{(k)} - \bar{\Theta}_n \right) \notag \\
 &  \quad\quad\quad\quad\quad\quad\quad  + \left(c/\sqrt{\log(\tau_{n+1} + 2)}\right)I_p \Delta W_n^{(k)}, \quad (\Theta_0^{(k)}, P_0^{(k)}) = (\theta_0^{(k)}, p_0^{(k)}), \notag\\
 & \quad\quad\quad\quad\quad \quad k =1,2, \ldots, K, \quad n=0,1,\ldots, N-1,   \label{Eq2.4} \\
 & \quad \text{with} \notag\\
 & \quad\quad\quad\quad\quad \bar{\Theta}_n = \sum\nolimits_{k=1}^{K} \pi_n^{(k)} \Theta_n^{(k)}, \quad \pi_n^{(k)} \ge 0, \quad \sum\nolimits_{k=1}^{K} \pi_t^{(k)} = 1,  \notag
\end{align}
where we have used the notation $\bigl(\Theta_{n}^{(k)}, P_{n}^{(k)}, \pi_n^{(k)} \bigr)=\bigl(\Theta_{\tau_n}^{(k)}, P_{\tau_n}^{(k)}, \pi_{\tau_n}^{(k)}\bigr)$ and the increments $\Delta W_n^{(k)} = (W_{n+1}^{(k)}-W_{n}^{(k)})$ are independent Gaussian random variables with mean $\mathbb{E} (\Delta W_n^{(k)})=0$ and variance $\operatorname{Var}(\Delta W_n^{(k)}) = \delta I_p$ (e.g., see \cite{r10} or \cite{r11}).

\subsection{Decision-making model for the principal} \label{S2.2} In what follows, we formalize the decision-making paradigm that will allow the principal, at each time step $n = 0, 1, 2, \ldots, N-1$, to determine a mixing distribution $\pi_n =\bigl(\pi_n^{(1)}, \pi_n^{(2)}, \ldots, \pi_n^{(K)}\bigr)$ over strategies, with $\pi_n^{(k)} > 0$, for all $k \in \{1,2,\ldots, K\}$, and sum to one, i.e., $\sum\nolimits_{k=1}^K \pi_n^{(k)} = 1$. Moreover, such a mixing distribution represents a set of dynamically apportioning coefficients with respect to each of the agents how their current parameter estimates effectively performed. In particular, we allow the principal to associate to each agent a performance index measure $\rho_n^{(k)} \in [0,1]$ which is determined by the $k^{\rm th}$ agent's current estimate $\Theta_{n}^{(k)}$ in connection with the test dataset $Z^{(K+1)}$ as follows
\begin{align}
 \rho_n^{(k)} = 1 - \exp\left(-\frac{\mu}{m_{K+1}} \sum\nolimits_{i=1}^{m_{K+1}} {\ell}\left(h_{\Theta_{n}^{(k)}}(x_i^{(K+1)}), y_i^{(K+1)}\right) \right), ~~ n &= 0, 1, 2, \ldots, N-1, \notag \\
                                                                                                                                      & \mu > 0. \label{Eq2.5}
\end{align}
where such a performance index is simply an empirical measure, with values between $0$ and $1$, quantifying the lack-of-fit between the model and the test dataset for an appropriately chosen loss function $\ell$ (e.g., see \cite{r12} for similar discussions).

Then, the principal incurred an average (or mixture) loss at each time step which is given by  
\begin{align}
 L_n = \sum\nolimits_{k=1}^K \pi_n^{(k)} \rho_n^{(k)}, \quad n = 0,1, 2, \ldots, N-1. \label{Eq2.6}
\end{align}
Then, the objective here is to present a dynamic allocation strategy, that has a decision-theoretic interpretation, guaranteeing an upper bound for the total overall mixture loss $L$, i.e.,
\begin{align}
 \min \to L &= \sum\nolimits_{n=0}^{N-1} L_n \notag \\
                 &= \sum\nolimits_{n=0}^{N-1}\sum\nolimits_{k=1}^K \pi_n^{(k)} \rho_n^{(k)}. \label{Eq2.7}
\end{align}
Moreover, such a dynamic allocation strategy will ensure the following additional property:
\begin{enumerate} [$\quad$]
 \item[$\quad$] {\it If the performance index measures $\rho_n^{(k)}$, for all $k \in \{1,2, \ldots, K \}$, tend to $0$, as $n \to \infty$. Then, the mean estimate $\bar{\Theta}_{n} = \sum\nolimits_{k=1}^K \pi_n^{(k)} \Theta_{n}^{(k)}$ tends to the optimal parameter estimate $\theta^{\ast} \in \Gamma \subset \mathbb{R}^p$ as $n \to \infty$.}
\end{enumerate}
In order to accomplish the above property, we use a simple dynamic allocation strategy coupled with an iterative updating scheme for determining the mixing distribution $\pi_{n}^{(k)}$, i.e.,
\begin{align}
 \pi_{n}^{(k)} = \frac{\alpha_{n}^{(k)}}{\sum\nolimits_{k=1}^K \alpha_{n}^{(k)}}, \quad k=1,2, \ldots, K, \label{Eq2.8}
\end{align}
while the weighting coefficients are updated according to
\begin{align}
 \alpha_{n+1}^{(k)} = \alpha_{n}^{(k)} \exp \left(-\rho_{n}^{(k)} \log(1/\beta)\right), \quad \beta \in (0,1), \quad n = 0, 1, 2, \ldots, N-1,  \label{Eq2.9}
\end{align}
where such a mixing distribution $\pi_{n}^{(k)}$ can be interpreted as a measure of quality corresponding to each of agents' current parameter estimates in connection with the test dataset. Note that we can assign the initial weighting coefficients $\alpha_{0}^{(k)}$, for $k=1,2, \ldots, K$, arbitrary values, but must be nonnegative and sum to one.

In what follows, if both the principal and the $K$ agents act in accordance with the decision-making/parameter-updating models outlined here (see also Subsection~\ref{S2.1}). Then, following proposition can be used to characterize the upper bound for the total overall mixture loss incurred to the principal.

\begin{proposition} \label{P1}
Let $L_n = \sum\nolimits_{k=1}^K \pi_n^{(k)} \rho_n^{(k)}$, $n = 0,1, 2, \ldots, N-1$, be a sequence of losses incurred to the principal, which is associated with the mixing distribution strategy $\pi_n =\bigl(\pi_n^{(1)}, \pi_n^{(2)}, \ldots, \pi_n^{(K)}\bigr)$ and the performance indices $\rho_n^{(k)}$, for $k=1,2,\ldots, K$, with respect to each of the agents' current parameter estimates. Then, the total overall mixture loss $L$ incurred to the principal, i.e., $L = \sum\nolimits_{n=0}^{N-1}\sum\nolimits_{k=1}^K \pi_n^{(k)} \rho_n^{(k)}$, satisfies the following upper bound condition. 
\begin{align}
L \le -\frac{1}{1-\beta} \log \left(\sum\nolimits_{k=1}^K  \alpha_{N}^{(k)} \right).  \label{Eq2.10}
\end{align}
\end{proposition}

\begin{proof} Note that, for any $\beta \in (0,1)$ and $\rho_n^{(k)} \in [0, 1]$, we have the following inequality relation\footnote{Note that, for any $k \in \{1,2,\ldots, K\}$, as $\rho_{n}^{(k)} \to 0$, we have the following relation
\begin{align*}
\beta^{\rho_{n}^{(k)}} & = \exp \left(-\rho_{n}^{(k)} \log(1/\beta)\right)\\
                                   &\approx 1- \rho_n^{(k)}\log(1/\beta) + o(\rho_n^{(k)}).
\end{align*}}
\begin{align}
 \exp \left(-\rho_{n}^{(k)} \log(1/\beta)\right) \le 1 - (1- \beta)\rho_n^{(k)}, \quad k=1,2,\ldots, K,\label{Eq2.11}
\end{align}
which is due to the convexity argument for $\exp \left(-\rho_{n}^{(k)} \log(1/\beta)\right)\equiv\beta^{r_n(k)} \le 1 - (1 - \beta)\rho_n^{(k)}$. Then, if we combine Equations~\eqref{Eq2.8} and \eqref{Eq2.9}, we will have
\begin{align}
\sum\nolimits_{k=1}^{K} \alpha_{n+1}^{(k)} &= \sum\nolimits_{k=1}^{K} \alpha_{n}^{(k)} \exp \left(-\rho_{n}^{(k)} \log(1/\beta)\right) \notag\\
                                                                   &\le \sum\nolimits_{k=1}^{K} \alpha_{n}^{(k)} \left(1 - (1- \beta)\rho_n^{(k)} \right) \notag \\
                                                                   & =  \sum\nolimits_{k=1}^{K} \alpha_{n}^{(k)} -  (1- \beta) \sum\nolimits_{k=1}^{K} \alpha_{n}^{(k)} \rho_n^{(k)} \notag \\
                                                                   & = \left(\sum\nolimits_{k=1}^{K} \alpha_{n}^{(k)}\right)  \left(1 - (1- \beta) \sum\nolimits_{k=1}^K \pi_n^{(k)} \rho_n^{(k)} \right). \label{Eq2.12}
\end{align}
Moreover, if we apply repeatedly for $n=0,1,2, \ldots, N-1$ to the above equation, then we have
\begin{align}
\sum\nolimits_{k=1}^{K} \alpha_{N}^{(k)} &\le \prod\nolimits_{n=0}^{N-1}\left(1 - (1- \beta) \sum\nolimits_{k=1}^K \pi_n^{(k)} \rho_n^{(k)} \right) \notag \\
                                                                &\le \exp \left( - (1- \beta) \sum\nolimits_{n=0}^{N-1}\sum\nolimits_{k=1}^K \pi_n^{(k)} \rho_n^{(k)} \right). \label{Eq2.13}
\end{align}
Notice that, due to the inequality $1 + t \le \exp(t)$ for all $t \in [0, \infty)$, the right-hand side of the above equation satisfies the following inequality
\begin{align}
 1 - (1- \beta) \sum\nolimits_{n=0}^{N-1}\sum\nolimits_{k=1}^K \pi_n^{(k)} \rho_n^{(k)} \le \exp \left( - (1- \beta) \sum\nolimits_{n=0}^{N-1}\sum\nolimits_{k=1}^K \pi_n^{(k)} \rho_n^{(k)}\right), \label{Eq2.14}
\end{align}
that further gives us the following result
\begin{align}
\sum\nolimits_{k=1}^{K} \alpha_{N}^{(k)} \le \exp \left( - (1- \beta) \sum\nolimits_{n=0}^{N-1}\sum\nolimits_{k=1}^K \pi_n^{(k)} \rho_n^{(k)} \right). \label{Eq2.15}
\end{align}
Hence, the statement in the proposition follows immediately.
\end{proof}

Here, it is worth mentioning that due to the inherent feedbacks and cooperative behavior among the agents, the collaborative learning offers some advantages in terms of stability and generalization, despite that both the principal and the agents do not necessarily need to have any knowledge of the sample distributions or the quality of each others' datasets. 

\subsection{Algorithm} \label{S2.3} In what follows, we present a generic algorithm for a collaborative learning with principal-agent setting. Note that Equations~\eqref{Eq2.8} and \eqref{Eq2.9} provide an effective method for the principal how to determine a set of aggregation coefficients based on how the agents' current parameter estimates effectively performed in connection with the testing dataset, whereas the agents, as part of their parameter estimate updates, use such aggregation coefficients in their parameter updating model of Equation~\eqref{Eq2.4} (i.e., in their mean-field-like interaction term) that eventually leading the agents to reach a consensus optimal parameter estimate. Figure~\ref{Fig1} depicts the visualization of the principal-agent relationship, where the process of aggregation is realized as a feedback mechanism/collaborative behavior through the mean-like-effect interaction among the agents. Moreover, the following generic algorithm provides a viable way to implement such a collaborative learning problem. 

{\rm \footnotesize

{\bf ALGORITHM:} Collaborative Learning with Principal-Agent Setting
\begin{itemize}
\item[{\bf Input:}] $K+1$ datasets, each with data size $m_k$: $Z^{(k)} = \bigl\{ (x_i^{(k)}, y_i^{(k)})\bigr\}_{i=1}^{m_k}$, for $k=1,2, \ldots, K+1$; an equidistant discretization time $\delta=\tau_{n+1} - \tau_n = T/N$, for $n=0,1,2, \ldots, N-1$, with $0=\tau_0 < \tau_1< \ldots < \tau_n<\ldots<\tau_N=T$, of the time interval $[0,T]$; Parameters: $\beta \in (0,1)$, $\gamma > 0$, $\eta > 0$, and $\mu > 0$.
\item[{\bf 0.}] {\bf Initialize:} Start with $n=0$, and set $\pi_0(k) =\alpha_0(k)= 1/K$, $(\Theta_0^{(k)}, P_0^{(k)}) = (\theta_0,p_0)$ for all $k=1,2,\ldots, K$.
\item[{\bf 1.}] {\bf Agents:} Parameter-updating step ({\bf see Equation~\eqref{Eq2.4}})
\begin{align*}
\bar{\Theta}_{n} &= \sum\nolimits_{k=1}^K \pi_{n}(k) \Theta_{n}^{(k)}\\
 \Theta_{n+1}^{(k)} &= \Theta_{n}^{(k)} + \delta P_{n}^{(k)} \\
 P_{n+1}^{(k)} &= (1 - \delta \gamma) P_{n}^{(k)} - \delta \nabla J_k(\Theta_n^{(k)},\hat{Z}^{(k)}) - \delta \eta \left(\Theta_n^{(k)} - \bar{\Theta}_n \right) \\
 &  \quad\quad\quad\quad\quad\quad\quad  + \left(c/\sqrt{\log(\tau_{n+1} + 2)}\right)I_p \Delta W_n^{(k)}, \quad (\Theta_0^{(k)}, P_0^{(k)}) = (\theta_0^{(k)}, p_0^{(k)})
  \end{align*}
\item[{\bf 2.}] {\bf Principal:} Decision-making step ({\bf see Equations~\eqref{Eq2.5}, \eqref{Eq2.8} and \eqref{Eq2.9}})
 \begin{itemize}
\item[{\bf i.}] Compute the performance index measure associate with each agent's parameter estimate ({\bf see Equation~\eqref{Eq2.5}})
\begin{align*}
\rho_n^{(k)} = 1 - \exp\left(-\frac{\mu}{m_{K+1}} \sum\nolimits_{i=1}^{m_{K+1}} {\ell}\left(h_{\Theta_{n}^{(k)}}(x_i^{(K+1)}), y_i^{(K+1)}\right) \right)
\end{align*}
 \item[{\bf ii.}] Update the weighting coefficients ({\bf see Equation~\eqref{Eq2.9}})
\begin{align*}
\alpha_{n+1}^{(k)} = \alpha_{n}^{(k)} \exp \left(-\rho_{n}^{(k)} \log(1/\beta)\right)
\end{align*}
 \item[{\bf iii.}] Update the aggregation coefficients ({\bf see Equation~\eqref{Eq2.8}})
\begin{align*}
  \pi_{n+1}^{(k)} = \frac{\alpha_{n+1}^{(k)}}{\sum\nolimits_{k=1}^K \alpha_{n+1}^{(k)}}
\end{align*}
 \end{itemize}
 for $k=1,2, \ldots, K$.
 \item[{\bf 3.}] Increment $n$ by $1$ and, then repeat Steps $1$ and $2$ until convergence, i.e., $\Vert \bar{\Theta}_{n+1} - \bar{\Theta}_{n}\Vert \le {\rm tol}$, or $n=N-1$.
 \item[{\bf Output:}] An optimal parameter value $\bar{\Theta}_{N} = \theta^{\ast}$.
\end{itemize}}
\begin{figure}[h]
\begin{center}
 \includegraphics[scale=0.9]{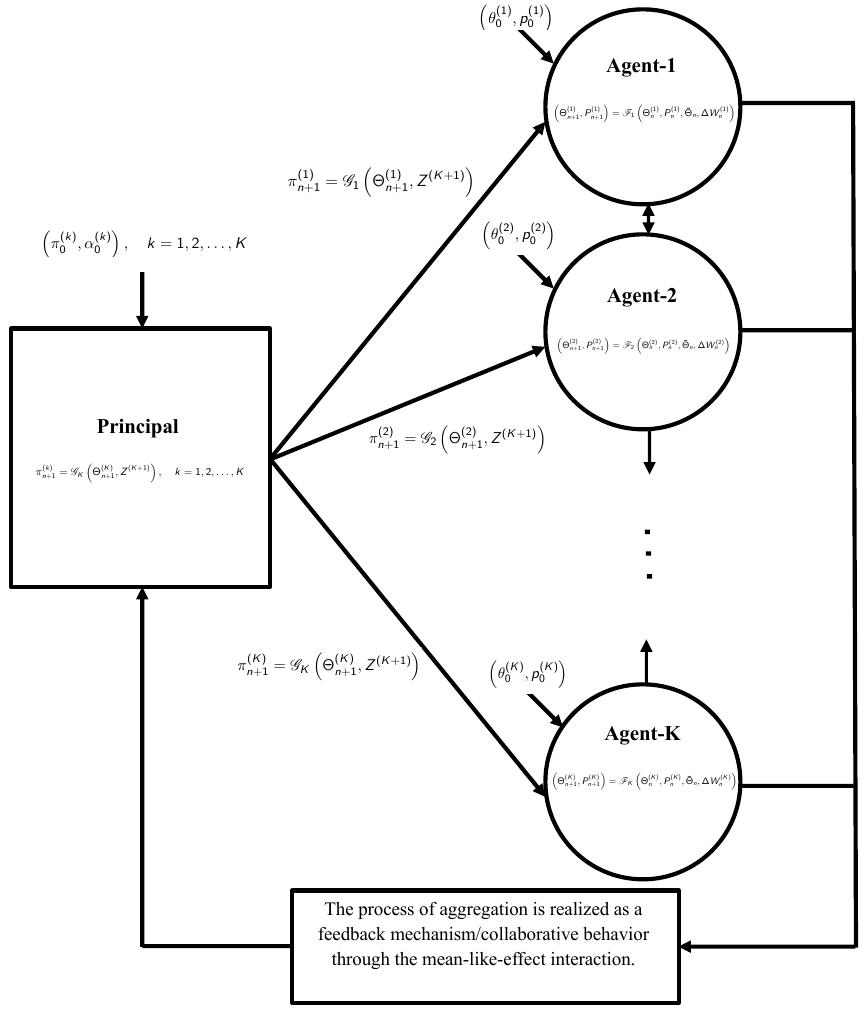}
  \caption{Visualization of the principal-agent relationship.} \label{Fig1}
{\scriptsize \begin{align*}
   {\rm \bf \text Notation:} \quad &{\rm Equation}~\eqref{Eq2.4} \Leftrightarrow \left(\Theta_{n+1}^{(k)}, P_{n+1}^{(k)}\right) = \mathscr{F}_k \left(\Theta_{n}^{(k)}, P_{n}^{(k)}, \bar{\Theta}_{n}, \Delta W_{n+1}^{(k)}\right),  \quad k=1,2, \ldots, K\\
 & {\rm Equations}~\eqref{Eq2.5}, ~\eqref{Eq2.8} ~ {\rm and} ~\eqref{Eq2.9} \Leftrightarrow \pi_{n+1}^{(k)} = \mathscr{G}_k\left(\Theta_{n+1}^{(k)}, Z^{(K+1)}\right), \quad k=1,2, \ldots, K
 \end{align*}}
 \end{center}
\end{figure} 

 \section{Numerical results and discussions} \label{S3} In this section, we presented numerical results for a simple nonlinear regression problem. In our simulation study, the numerical data for the population of {\it Paramecium caudatum}, which is a species of unicellular organisms, grown in a nutrient medium over $24$ days (including the starting day of the experiment), were digitized using the Software: WebPlotDigitizer \cite{r15} from the figures in the paper by F.G. Gause \cite{r16} (see also \cite[pp.~102]{r17}). Here, our interest is to estimate the parameters for the population growth model, on the assumption that the model obeys the following logistic law, i.e.,
 \begin{align*}
  N_{\theta}(t) = \frac{ N_0\,N_e}{N_0 + (N_e - N_0)\exp(-r t)}, \quad \theta = (N_0, N_e, r),
\end{align*}
 where $N_{\theta}(t)$ is the number of {\it Paramecium caudatum} population at time $t$ in $[{\rm Days}]$, and $N_0$, $N_e$ and $r$ (i.e., $\theta = (N_0, N_e, r)$) are the parameters to be estimated using the digitized dataset obtained from Gause's paper, i.e., the original dataset $Z = \bigl\{(t_i, N_i)\bigr\}_{i=1}^{24}$, with a total of $24$ digitized dataset points, $t_i$ is the time in $[{\rm Days}]$ and $N_i$ is the corresponding number of {\it Paramecium caudatum} population. We specifically used {\rm One-Principal} and {\rm Two-Agents} in our collaborative learning problem (i.e., the number of agents is $K=2$). Moreover, we partitioned the original dataset into three separate subsets as follows:
 \begin{enumerate} [(i)]
 \item The datasets corresponding from ${\rm Day}$-1 to ${\rm Day}$-8 for the Agent-1, with data size of $m_1=8$, and from ${\rm Day}$-16 to ${\rm Day}$-24 for the Agent-2, with data size of $m_2=9$.
 \item The dataset corresponding from ${\rm Day}$-9 to ${\rm Day}$-15 for the Principal, with data size of $m_3=7$. 
 \end{enumerate}
 Note that we used a simple Euler--Maruyama time discretization approximation scheme to solve numerically the corresponding system of SDEs (cf. Equation~\eqref{Eq2.3}), with an equidistance discretization time $\delta = 1 \times 10^{-5}$ of the time interval $[0,1]$. For both model training and generalization processes, we used the usual quadratic loss function 
 \begin{align*}
 J_k(\theta^{(k)}, Z^{(k)})=(1/m_k)\sum\nolimits_{i=1}^{m_k} \left(N_{\theta}(t_i^{(k)}) - N_i^{(k)}\right)^2, \quad k=1,2,3,
\end{align*}
that quantifying the lack-of-fit between the model and the corresponding datasets. Figures~\ref{Fig2} shows both the digitized original dataset from Gause's paper and the population growth model $N(t)$, with optimal consensus parameter values $N_0^{\ast} = 1.1224$, $N_e^{\ast} = 229.9285$ and $r^{\ast} =0.7259$, versus time $t$ in $[{\rm Days}]$ on the same plot. In our simulation, we used the following numerical values for: $\beta = 0.5$,  $\gamma = 1$, $\eta = 0.01$, $\mu=0.001$, and a noise level of $c = 0.001$ (i.e., satisfying $\beta \in (0,1)$, $\gamma > 0$, $\eta > 0$ and $\mu > 0$, see also the Algorithm in Subsection~\ref{S2.3}). 
\begin{figure}[h]
\begin{center}
 \includegraphics[scale=0.195]{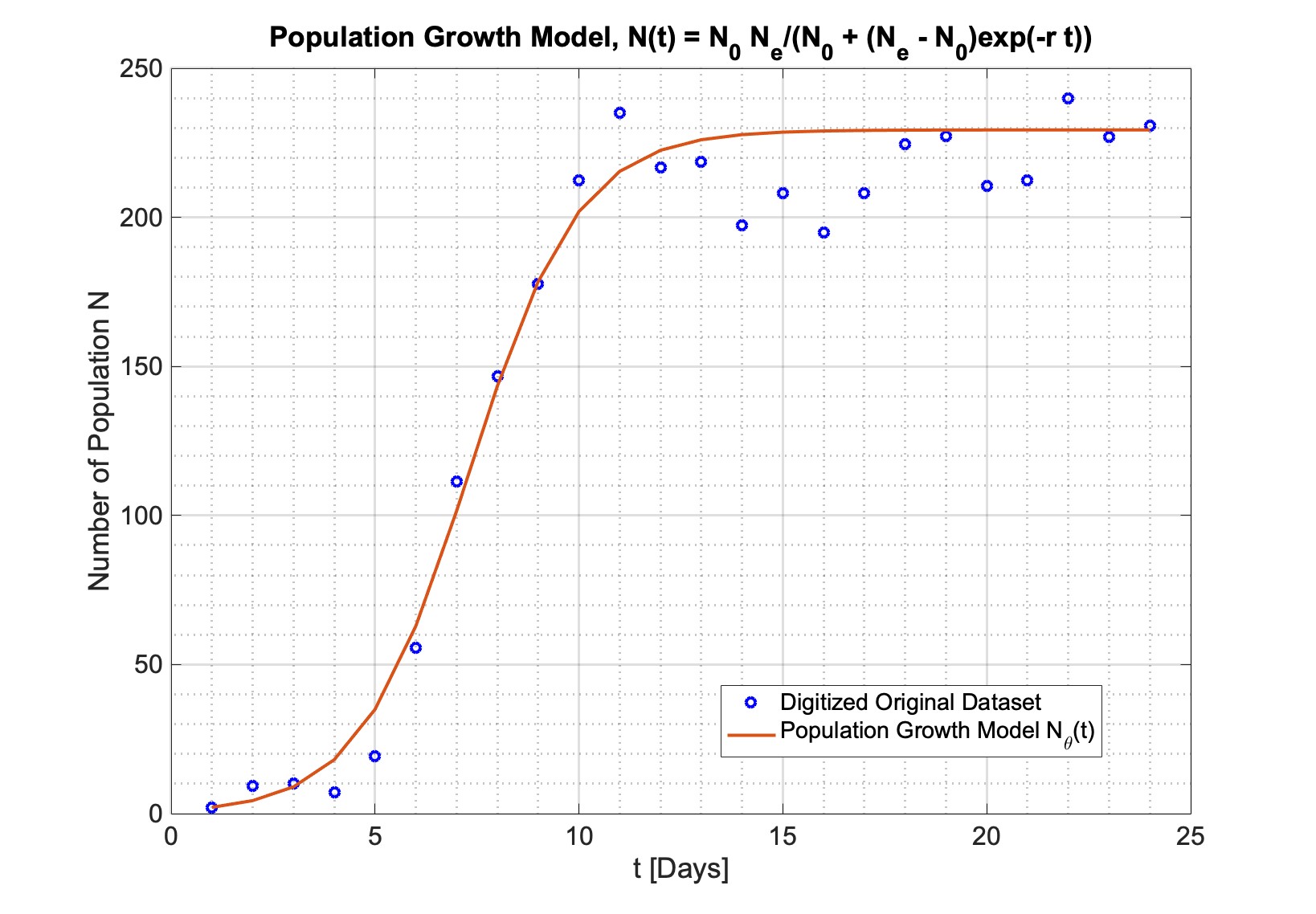}
  \caption{Plots for the original dataset and the population growth model.} \label{Fig2}
\end{center}
\end{figure} 
Moreover, Figure~\ref{Fig3} shows the evolutions for the mean parameter estimates. Here, we can see that the proposed learning framework, with principal-agent setting, allowed us to determine optimal consensus parameter values for the model parameters $N_0$, $N_e$ and $r$, despite that the sample distributions of the datasets are are quite different for both the principal and the agents (see also Figure~\ref{Fig2}). Finally, it is worth remarking that such collaborative learning framework could be interesting to investigate from game-theoretic perspective (e.g., see \cite{r13} and \cite{r14} for an interesting study in computer science literature).
\begin{figure}[h]
\begin{center}
 \includegraphics[scale=0.175]{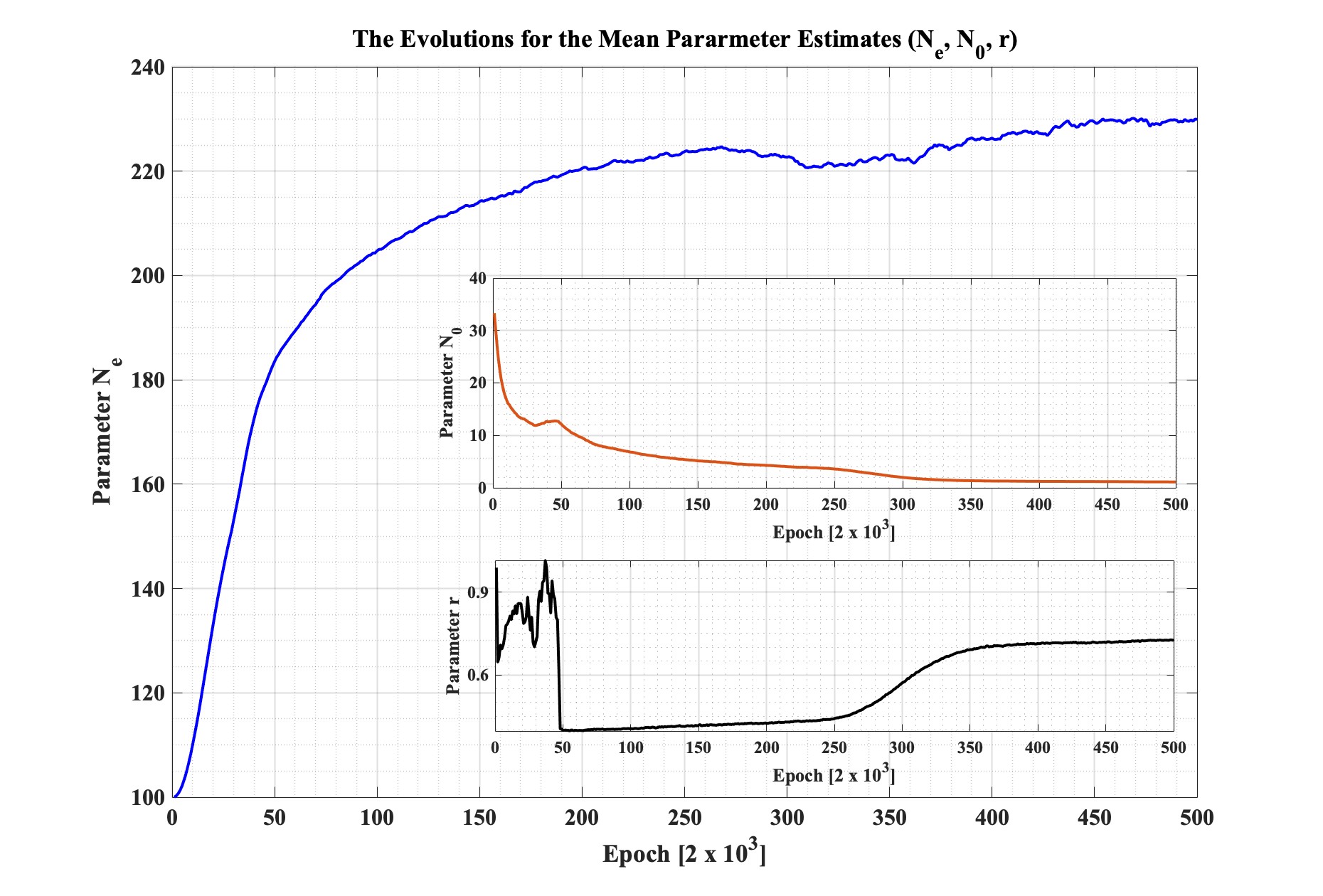}
  \caption{Plots for the evolution of the mean parameter estimates.} \label{Fig3}
\end{center}
\end{figure}

\end{document}